\documentclass[twoside,11pt]{article}

\usepackage{blindtext}

\usepackage[abbrvbib,preprint]{jmlr2e}

\usepackage{algorithm}
\usepackage{amsmath,amssymb,amsfonts,mathrsfs}
\usepackage{breakcites}
\usepackage{graphicx}
\usepackage{soul}
\usepackage{pdfpages}
\usepackage{bm}
\usepackage{array}
\usepackage{multirow}
\usepackage{booktabs}  
\usepackage{lipsum}
\usepackage{makecell}
\usepackage{rotating}
\usepackage{tablefootnote}

\allowdisplaybreaks

\usepackage{threeparttable}

\newcommand{\gd}{\text{\tiny GD}}
\newcommand{\sgd}{\text{\tiny SGD}}

\newcommand{\nfrac}[2]{{#1} / {#2}}

\newcommand{\vone}{{\bm 1}}

\newcommand{\ve}{{\bm e}}

\newcommand{\vpi}{\bm \pi}

\newcommand{\bbE}{\mathbb{E}}

\newcommand{\bbR}{\mathbb{R}}

\newcommand{\cA}{{\mathcal A}}

\newcommand{\cE}{{\mathcal E}}

\newcommand{\cO}{{\mathcal O}}

\newcommand{\cW}{{\mathcal W}}

\newcommand{\cZ}{{\mathcal Z}}

\newcommand{\bigO}[1]{O\left(#1\right)}
\newcommand{\bigOmega}[1]{\Omega\left(#1\right)}

\newcommand{\argmin}{{\arg\min}}

\newcommand{\pr}{\text{Pr}}

\newcolumntype{M}[1]{>{\centering\arraybackslash}m{#1}}

\jmlrheading{24}{2023}{1-\pageref{LastPage}}{04/03; Revised XX/XX}{XX/XX}{XX-XXXX}{Zhang, Teng and Zhang}
\usepackage{lastpage}

\ShortHeadings{Lower Generalization Bound in Smooth SCO}{Zhang, Teng and Zhang}
\firstpageno{1}

\begin{document}

\title{Lower Generalization Bounds for GD and SGD \\
in Smooth Stochastic Convex Optimization}

\author{\name Peiyuan Zhang
       \email peiyuan.zhang@yale.edu \\
       \addr Yale University
       \AND
       \name Jiaye Teng  \email tjy20@mail.tsinghua.edu.cn \\
       \addr Tsinghua University
       \AND
       \name Jingzhao Zhang\thanks{Corresponding author.} \email jingzhaoz@mail.tsinghua.edu.cn \\
       \addr Tsinghua University \& Shanghai Qizhi Institute
       }

\editor{My editor}

\maketitle

\begin{abstract}
This work studies the generalization error of gradient methods. More specifically, we focus on how training steps $T$ and step-size $\eta$ might affect generalization in \emph{smooth} stochastic convex optimization (SCO) problems. We first provide tight excess risk lower bounds for Gradient Descent (GD) and Stochastic Gradient Descent (SGD) under the general \emph{non-realizable} smooth SCO setting, suggesting that existing stability analyses are tight in step-size and iteration dependence, and that overfitting provably happens. 
Next, we study the case when the loss is \emph{realizable}, i.e. an optimal
solution minimizes all the data points. Recent works show better rates can be attained but the improvement is reduced when training time is long. Our paper examines this observation by providing excess risk lower bounds for GD and SGD in two \emph{realizable} settings: 1)  $\eta T = \bigO{n}$, and (2)  $\eta T = \bigOmega{n}$, where $n$ is the size of dataset. In the first  case $\eta T = \bigOmega{n}$, our lower bounds tightly match and certify the respective upper bounds. However, for the case $\eta T = \bigOmega{n}$, our analysis indicates a gap between the lower and upper bounds. A conjecture is proposed that the gap can be closed by improving upper bounds, supported by analyses in two special scenarios. 
\end{abstract}


\section{Introduction}

Gradient methods are the predominant algorithms for training neural networks. These methods are not only efficient in time and space, but more importantly, produce solutions that generalize well~\citep{he2016deep,vaswani2017attention}. Understanding why neural networks trained with gradient methods perform well on test data can be very challenging, as the phenomenon results from an interplay between the network architecture, the data distribution as well as the training algorithm~\citep{jiang2019fantastic,zhang2021understanding}. In this work, we aim to shed some light on the role of (stochastic) gradient descent, and take a humble step by considering generalization in smooth convex problems.  

Much work has been done for analyzing gradient descent in convex learning problems. The early approach exploits the convex structure and shows that gradient methods find approximate empirical risk minimizers. Then the generalization can be bounded by uniform convergence~\citep{shalev2014understanding}. However, this approach is limited in scalability to high dimensions and can be provably vacuous even for the simple task of linear regression~\citep{shalev2010learnability,feldman2016generalization}. An alternative approach \citep{nemirovskij1983problem} that addresses the high-dimension problem is via online-to-batch conversion. This approach achieves minimax sample complexity, but it only applies to single-pass training, whereas in practice, models trained for longer periods can generalize better~\citep{hoffer2017train}.

Several recent explanations have been proposed to bridge the gap and bound generalization in multi-pass settings~\citep{soudry2018implicit,ji2019implicit,lyu2021gradient,bartlett2020benign}. These works demonstrate that gradient descent benefits from implicit bias and finds max-margin solutions for classification problems, as well as min-norm solutions for regression problems. However, characterizing the implicit bias for other loss functions or non-linear models remains a challenging task.

One method that generalizes to a broader range of loss functions and models is the stability argument~\citep{bousquet2002stability,hardt2016train}. This argument shows that if the model and the training method is not overly sensitive to data perturbations, the generalization error can be effectively bounded. However, this argument suffers from a large number of training updates especially when the step-size is sufficiently large, while in practice, generalization often benefits from longer training time. 
\emph{It is unclear whether longer training truly hurts generalization in smooth convex learning problems, because the tightness of the growing upper bounds remains unknown, and might just result from an artifact of the analysis. }

In this work, we analyze the above problem in smooth \emph{stochastic convex optimization} (SCO). More specifically, we focus on how the \emph{training horizon} $\eta T$ might affect the generalization property: the product of the step size $\eta$ and the number of iteration $T$ is a better measure for training intensity as large number of iterations may not even train a model if $\eta$ is close to zero. 
While several recent works have established fast convergence rates in test error when $\eta T$ is not too large under the \emph{realizable} condition~\citep{lei2020fine,nikolakakis2022beyond,schliserman2022stability}, our work provides the first tight lower bounds in these scenarios and suggest that known bounds are tight in some settings but likely not when $\eta T$ is large.



\paragraph{Our contributions.}
Let $\eta$ represent the step-size in gradient methods, $T$ denote the iteration number, and $n$ denote the sample size. Our contributions are as follows and presented in Table~\ref{tab: summary},
\begin{itemize}
    \item We first provide a tight lower bound $\bigOmega{\frac{1}{\eta T} + \frac{\eta T}{n}}$ for the smooth non-realizable SCO. 
    \item For realizable SCO, we notice a gap between two types of analysis, as shown in Table~\ref{tab: summary}: 
      \begin{enumerate}
        \item when $\eta T = \bigO{n}$, we prove matching lower bounds for the excess population risk for GD and SGD under the smooth and realizable SCO setting; 
        \item  when  $\eta T = \bigOmega{n}$, we provide a lower bound construction that suggests a gap exists between upper and lower bound. 
      \end{enumerate}
    \item We conjecture that the upper bound when $\eta T = \bigOmega{n}$ is not tight for large step-sizes. We provide evidences for the conjecture in two special scenarios: (1) one-dimensional convex problems and (2) high-dimensional linear regression.
\end{itemize}

Our results offer insights and answers to the question of how longer training can impact generalization error. For non-realizable cases, our lower bound suggests that training for a longer time can provably lead to overfitting, even for smooth convex problems. For realizable cases, our lower bounds suggest that longer training might actually reduce the generalization error. Moreover, our new lower bounds in the smooth setting, compared with those known in the nonsmooth setting, suggest that smoothness and realizability together might explain why training longer does not lead to overfitting.



\paragraph{Notations.} 
For any positive integer $n$, we denote the set $[n] := \{1, 2, \dots, n\}$. $\| \cdot\|$ denotes the $l_2$ norm for vectors. We use $\text{Bern}(p)$ to denote the Bernoulli distribution with probability $p$ to be $1$ and $\text{Unif}(S)$ to denote the uniform distribution over set $S$. Occasionally we will use capital letters, i.e. $W$, to denote "large" vectors, in contrast to usual vectors like $v,w$. Also, Let $\vone$ be all-one vector and $\ve_i$ be the vector with a $1$ in the $i$-th coordinate and $0$ elsewhere. 
\section{Preliminaries} \label{sec: setting}
Following ... , We study the generalization error of \emph{stochastic convex optimization (SCO)} problem. We receive a dataset of finite samples $S = \{z_1, \dots, z_n\}$, where each $z_i$ is i.i.d. drawn from an unknown distribution $D$ over sample space $\cZ$ and $n$ is the size of dataset.  Our goal is to find a model parameterized by $w \in \cW \subseteq \bbR^d$ that minimizes the \emph{population} (or expected) risk over $D$, defined as:
\begin{equation}
    F(w) = \bbE_{z \sim D}[f(w,z)]
\end{equation}
where $f(w, z): \bbR^d \times \cZ \to \bbR$ is the loss function evaluated on a single example $z \in \cZ$.

Since the population loss $F$ is typically inaccessible, we instead employ an averaged substitute on sample $S$, known as the \emph{empirical} risk:
\begin{equation}
    F_S(w) = \frac{1}{n} \sum_{i=1}^n f(w,z_i), \qquad S \sim D^n.
\end{equation}

Given dataset $S$ and any (stochastic) algorithm $\cA$, we denote $\cA[S]$ as the output of running $\cA$ on the training sample $S$. In this paper, we are interested in bounding the \emph{excess population risk} of $\cA[S]$:
\begin{equation*}
    \bbE_{S,\cA}[F(\cA[S])] - \min_{w \in \cW} F(w),
\end{equation*}
where the expectation is taken over the randomness of sample $S$ and algorithm $\cA$.

\subsection{Gradient Methods}
In this work, we focus on understanding the excess risk for two simplified algorithms: Gradient Descent (GD) and Stochastic Gradient Descent (SGD). Gradient descent is one of the most well-known optimization methods. At iteration $t$, GD employs the following recurrence:
\begin{equation} \label{eq: gd}
    w_{t+1} = w_t - \eta \nabla F_S(w_t),  
\end{equation}
where $\eta > 0$ is the step-size and $\nabla F_S(w)$ is the average stochastic gradient on sample set $S$. We usually employ the time average $w_{\gd} := \bar{w}_T = \frac{1}{T} \sum_{t=1}^T w_t$ as the output of GD.  

In practice, many practitioners favor the Stochastic Gradient Descent (SGD) method over GD for its computational efficiency. In this work, we study standard SGD, i.e., in iteration $t \in [T]$,
\begin{equation} \label{eq: sgd}
    w_{t+1} = w_t - \eta \nabla f(w_t,z_{i_t}), 
\end{equation}
where $z_{i_t}$ is uniformly sampled from $S$ \emph{with replacement} as $i_t \sim \text{Unif}([n])$. 
The output for SGD is the average $w_{\sgd} := \bar{w}_T = \frac{1}{T} \sum_{t=1}^T w_t$.  

\subsection{Smooth Stochastic Convex  Optimization}
In order to derive non-vacuous bounds on the excess risk for SCO, we make assumptions on the properties of $f(w, z)$. First, we assume the access to the value $f(w,z)$ and the unbiased stochastic gradient estimator $\nabla f(w, z)$ for any $w \in \cW$ and $z \in \cZ$. 

When the function is nonsmooth, this problem has been extensively studied \citep{bassily2020stability}, and known rates were proven to be optimal \citep{amir2021sgd,sekhari2021sgd,nemirovskij1983problem}.
However, less is known when the function is differentiable and smooth. Indeed, while upper bounds have been well-established in literature \citep{hardt2016train,lei2020fine,nikolakakis2022beyond}, the optimality of these results are yet to be certified by corresponding lower bounds. In this work, we aim to provide lower bounds for the smooth SCO setting and make the following assumptions.

\begin{definition} \label{def: smooth}
$f(w, z)$ is $L$-smooth if it satisfies $\| \nabla f(w_1, z) - \nabla f(w_2, z) \| \leq L \| w_1 - w_2 \|$ for any $w_1, w_2$ and $z \in \cZ$.
\end{definition}


\begin{definition} \label{def: convex}
$f(w, z)$ is convex if it satisfies $f(w_1, z) \geq f(w_2, z) + \langle w_1 - w_2, \nabla f(w_2, z)\rangle$ for any $w_1, w_2$ and $z \in \cZ$.
\end{definition}

\begin{table}[t]
\centering
\begin{threeparttable}
\begin{tabular}{ M{0.4cm} M{2.3cm}  M{3.5cm}  M{2.8cm} M{2.1cm} M{2.0cm}  }
\toprule
& & & & \\[-0.3cm]
& & GD & SGD & Best Sample Complexity & Overfitting for large $\eta T$\\[0.1cm] 
 \midrule 
& & & & \\[-0.2cm]
\multirow{4}{*}{\begin{sideways}Non-realizable\end{sideways}}
\multirow{4}{*}{
}
& Upper bound & \thead{{\normalsize	 $\bigO{\frac{1}{\eta T} + \frac{\eta T}{n} }$} \\{\footnotesize\citep{hardt2016train}}}  & \thead{{\normalsize	 $\bigO{\frac{1}{\eta T} + \frac{\eta T}{n} }$} \\{\footnotesize\citep{hardt2016train}}} & $\bigO{1/\sqrt{n}}$ & Yes  \\[0.3cm]
& & & & \\[-0.1cm]
& Lower bound & \thead{{\normalsize	 $\bigOmega{\frac{1}{\eta T} + \frac{\eta T}{n} }$} \\ {(\footnotesize Theorem.~\ref{thm: lb-sco-gd})}} & \thead{{\normalsize	 $\bigOmega{\frac{1}{\eta T} + \frac{\eta T}{n} }$} \\{(\footnotesize Theorem.~\ref{thm: lb-sco-sgd})}} & $\bigOmega{1/\sqrt{n}}$ & Yes \\[0.4cm]
\midrule 

& & & & \\[-0.2cm]
\multirow{7}{*}{\begin{sideways}Realizable \end{sideways}}
\multirow{7}{*}{
}
\multirow{7}{*}{
}
& Upper bound  & $\bigO{\frac{1}{\eta T} + \frac{1}{n} + \frac{\eta T}{n^2} }$\tnote{$\dag$}  {\footnotesize \citep{nikolakakis2022beyond}} & $\bigO{\frac{1}{\eta T} + \frac{\eta}{n} + \frac{\eta T}{n^2} }$ {\footnotesize	 \citep{lei2020fine}} & $\bigO{1/n}$ & Yes \\[0.3cm]
& & & & \\[-0.1cm]
& Lower bound ($\eta T = O(n)$) & $\bigOmega{\frac{1}{\eta T} + \frac{1}{n} + \frac{\eta T}{n^2}}$ {(\footnotesize Theorem.~\ref{thm: lb-1})} & $\bigOmega{\frac{1}{\eta T} + \frac{1}{n} + \frac{\eta T}{n^2}}$ {(\footnotesize Theorem.~\ref{thm: lb-1})} & $\bigOmega{1/n}$ & N.A.\\[0.4cm] 

 & & & & \\[-0.1cm]
& Lower bound ($\eta T = \Omega(n)$) & $\bigOmega{\frac{1}{\eta T} + \frac{1}{n}}$ {(\footnotesize Theorem.~\ref{thm: lb-2})} & $\bigOmega{\frac{1}{\eta T} + \frac{1}{n} }$ {(\footnotesize Theorem.~\ref{thm: lb-2})} & $\bigOmega{1/n}$ & No \\[0.4cm] 
\bottomrule
\end{tabular}
\caption{\small Summary of our results. We present our lower bounds and compare with existing upper bounds. In particular, we split to three setting: non-realizable, realizable under $\eta T = \bigO{n}$ and realizable under $\eta T = \bigOmega{n}$. For each setting, we provide lower bounds for the excess risk of GD, SGD. We also provide the best possible sample complexity and whether overfitting happens when $\eta T$ is large.}
\label{tab: summary}
\begin{tablenotes}
\item[$\dag$] The bound in $\eta, T$ and $n$ is not explicitly stated in \citet{nikolakakis2022beyond}. For the expression, please refer to a derivation in Appendix~\ref{appendix: gd-ub}.
\end{tablenotes}
\end{threeparttable}
\vspace{-0.3cm}
\end{table}

\paragraph{Realizable smooth SCO.}
The smooth SCO problem can be divided into two cases depending whether the optimal solution minimizes all data points simultaneously. When this happens, it is usually referred as a \emph{realizable} setting, formally defined below. 
\begin{definition} \label{def: realizable}
We say that $f(w, z)$, $z \in \cZ$ formalizes a realizable setting if for any $z \in \cZ$ $$f(w^*, z) = \min_{w} f(w, z), \quad \text{where} \quad w^* = \argmin_{w}F(w).$$
\end{definition}
If $f(w,z)$ is smooth, convex, and realizable, we refer the setting as the \emph{realizable} smooth SCO. The realizable condition implies immediately the property called \emph{weak growth condition}, which is stated in the following lemma.
\begin{lemma} \label{lemma: growth-condition} If $f(w, z), z \in \cZ$ is realizable and $L$-smooth, then for any $w, z \in \cZ$ it holds 
    \begin{equation*}
        \| \nabla f(w, z) \|^2 \leq 2L \left( f(w, z)  - f(w^*, z)\right).
    \end{equation*}
\end{lemma}
The growth condition connects the rates at which the stochastic gradients shrink relative to its value. It is widely employed in stochastic optimization literature to improve the convergence rate of SGD and GD under overparameterized or realizable setting \citep{vaswani2019fast}. Recent papers \citep{lei2020fine,schliserman2022stability,nikolakakis2022beyond} focused on the generalization bound under realizable smooth SCO also suggest that such an assumption improves the sample complexity upper bounds.

\paragraph{Non-realizable smooth SCO.} 
We say a convex learning problem is non-realizable if it does not satisfy the condition in Definition~\ref{def: realizable}. In this setting, known upper bounds for sample complexity actually yield a slower convergence rate at $\bigO{1/\sqrt{n}}$~\citep{hardt2016train} when we set $\eta T = \Theta(\sqrt{n})$, which suggests that longer training leads to overfit.

Next, we discuss our main results in the settings introduced above, and explain why a gap exists for the realizable setting when $\eta T$ is large.

\section{Overview of Results: Training Horizon $\eta T$ and Overfitting}
In this section, we present and discuss our main result on generalization lower bounds. Also, we give a comparison with existing upper bounds. 

Before proceeding to our lower bounds, we motivate the study of the role of the training horizon $\eta T$. Intuitively, compared with the number of iterations $T$, the product of the step-size $\eta$ and $T$ provides a more accurate measure of the intensity of the training process. This is because an arbitrarily small step-size with a large $T$ does not necessarily decrease the optimization error to convergence. Moreover, the importance of $\eta T$ is showcased by existing generalization upper bounds for SGD: \citet{lei2020fine} established the first excess risk upper bound for SGD under realizable smooth SCO:
\begin{equation} \label{eq: ub-sgd-tn}
     \bbE[F(w_{\sgd})] - \min_{w \in \cW} F(w) = \bigO{\frac{1}{\eta T} + \frac{\eta}{n} + \frac{\eta T}{n^2}}.
\end{equation}
The upper bound suggests the generalization error will first decrease and then increase when training horizon $\eta T$ becomes large. Compared with the non-realizable case, a fast rate $O(1/n)$ of sample complexity can be obtained only if training horizon satisfies $\eta T = O(n)$. Otherwise, if $\eta T$ is sufficiently large, say $\eta T = n^2$, overfitting will happen and the generalization error becomes $O(1)$. This similar to the overfitting behavior of non-realizable upper bound  $\bigO{\tfrac{1}{\eta T} + \tfrac{\eta T}{n} }$.

These upper bound results differ from the empirical observation that often longer training helps generalization. To bridge the gap between theory and practice, in this work, we analyze the relationship between overfitting, generalization error and training horizon $\eta T$ from a \emph{lower bound} perspective. Our lower bound construction indicates the generalization error has different regimes depending on the horizon $\eta T$: when $\eta T = O(n)$, the lower bound for SGD is (per Theorem~\ref{thm: lb-1})
\begin{equation}
     \bbE[F(w_{\sgd})] - \min_{w \in \cW} F(w) = \bigOmega{\frac{1}{\eta T} + \frac{1}{n} + \frac{\eta T}{n^2}},
\end{equation}
whereas when $\eta T = \Omega(n)$, we have (per Theorem~\ref{thm: lb-2})
\begin{equation}
     \bbE[F(w_{\sgd})] - \min_{w \in \cW} F(w) = \bigOmega{\frac{1}{\eta T} + \frac{1}{n}}.
\end{equation}
Our lower bound result suggests overfitting will \emph{not} happen, and the sample complexity $\Omega(1/n)$ can be achieved even when the training horizon $\eta T$ is sufficiently large. This is in contrast with the existing upper bounds, nevertheless, corresponds to the empirical results. To bridge the gap, we conjecture that the upper bound can be improved under the setting $\eta T = \Omega(n)$, and provide evidences in Section~\ref{sec: infinite}. Similar conclusion also applies to GD.

Moreover, in this work, we provide novel results on the lower bounds of both GD and SGD in the non-realizable setting, which to the best of our knowledge have not been previously reported. Specifically, these non-realizable lower bounds tightly match the existing upper bounds, which suggests that overfitting always occurs when the training horizon $\eta T$ is sufficiently large under the non-realizable condition. We report our lower bound results and compare them with the corresponding upper bounds in Table~\ref{tab: summary}. To give a sense of the role of realizable condition in improving generalization, we include the best possible sample complexity under each setting. Additionally, to illustrate the role of $\eta T$, we indicate whether overfitting will occur as $\eta T$ tends to infinity.

\section{Main Theorems: Lower Bounds in Smooth SCO} \label{sec: lb}
In this section we present our main results on lower bounds for the smooth SCO.
We aim to reduce the gaps mentioned in Section~\ref{sec: setting} for all the above settings. To this end, we will split our discussion to three parts: (1) non-realizable, (2) realizable with $\eta T = \bigO{n}$ and (3) realizable with $\eta T = \bigOmega{n}$.

\subsection{Non-realizable Setting} \label{sec: nonrealizable}
We first discuss the non-realizable setting and provide a novel lower bound for the excess risk of GD in the following theorem.
\begin{theorem} \label{thm: lb-sco-gd}
    For any $\eta > 0$, $T > 1$ with $1/T \leq \eta = \bigO{1}$\footnote{This is a mild condition because (1) step-size cannot exceed $\bigO{1}$, in order to make the optimization method converge for $\bigO{1}$-smooth function, (2) an overly small step-size $\eta$ cannot even guarantee the convergence in the optimization sense and $T$ is arbitrarily large to ensure $\eta T \geq 1$. We will assume this holds in the statement of rest theorems and lemmas.}, there exists a convex, $1$-smooth $f(w,z): \bbR \to \bbR$ for every $z \in \cZ$, and a distribution $D$ such that, with a bounded initialization $\| w_1 - w^* \| = \bigO{1}$, the output $w_{\gd}$ for GD satisfies
    \begin{align*}
        \bbE[F(w_T)] - F(w^*) = \bigOmega{\frac{1}{\eta T} + \frac{\eta T}{n}}.
    \end{align*}
\end{theorem}
The lower bound in Theorem~\ref{thm: lb-sco-gd} tightly matches the corresponding upper bound established in \cite{hardt2016train} (see Table~\ref{tab: summary}). It can be translated to a lower bound of sample complexity: for any $T > 1$, by setting $\eta = \sqrt{n}/T$, we derive a $\bigOmega{1/\sqrt{n}}$ bound which certifies the optimality of existing upper bound  $\bigO{1/\sqrt{n}}$. To the best of our knowledge, this is the first such result for GD. A recent work provides lower bound for the uniform stability of (S)GD \citep{zhang2022stability} under smooth SCO, but it does not directly imply a bound on the excess risk.

The key step in the proof of Theorem~\ref{thm: lb-sco-gd} is to find a hard instance that gives an overfitting lower bound $\bigOmega{\eta T/n}$. To this end, we employ a technique inspired by Theorem 3 and Lemma 7 in \citet{sekhari2021sgd}: in non-realizable setting, the stochastic gradient does not necessarily scale down with the value of $f(w,z)$. As a result, by utilizing an \emph{anti-concentration} argument, we show that with non-vanishing probability $\bigOmega{1}$, the absolute value of $w_t$ increases by a rate of $\bigOmega{\eta/\sqrt{n}}$ in each step. Then, calculation suggests a $\bigOmega{\eta T/n}$ bound for the function value. The details can be found in Appendix~\ref{appendix: lb-sco-gd}. In the meanwhile, the term $\bigOmega{1/\eta T}$ reflects the optimization error and the proof is provided in Lemma~\ref{lemma: lb-suboptimality}, Appendix~\ref{appendix: lb-suboptimality}. 

A similar result holds for SGD in the following theorem.
\begin{theorem} \label{thm: lb-sco-sgd}
    For any $\eta > 0$, $T > 1$, there exists a convex, $1$-smooth $f(w,z): \bbR \to \bbR$ for every $z \in \cZ$, and a distribution $D$ such that, with a bounded initialization $\| w_1 - w^* \| = \bigO{1}$, 
    the output $w_{\sgd}$ for SGD satisfies
      \begin{align*}
        \bbE[F(w_{\sgd})] - F(w^*) = \bigOmega{\frac{1}{\eta T} + \frac{\eta T}{n}}.
    \end{align*}
\end{theorem}
This also matches the SGD upper bound in \cite{hardt2016train} and implies a sample complexity bound $\bigOmega{1/\sqrt{n}}$ if we set $\eta = \sqrt{n}/T$ for any $T$.  

We emphasize the bound for SGD is novel compared with existing works: it is a forklore that in \citet{nemirovskij1983problem}, single-pass SGD ($ T = n $) achieves a sample complexity lower bounds for Lipschitz convex functions (where a smooth function within a bounded domain is automatically Lipschitz). Yet, our result is the first to provide an explicit dependence on $T$ and $\eta$ and applies to an arbitrary number of updates. It shows that training longer can provably lead to overfitting, and answers the question raised in the introduction for the non-realizable setting.

\subsection{Realizable Setting: \texorpdfstring{$\eta T = O(n)$}{eta T=O(n)}} \label{sec: t_equal_n}
Better stability-based generalization upper bound are known for realizable problems. However, we will see that our lower bounds suggest known results may not be tight. In this subsection we first provide our lower bounds for the realizable setting when condition $\eta T = \bigO{n}$ is satisfied. The next theorem  characterizes the lower bounds for GD and SGD.
\begin{theorem} \label{thm: lb-1}
For every $\eta > 0$, $T > 1$, if condition $T = \cO(n)$ holds, then there exists a convex, $1$-smooth and realizable $f(w, z): \bbR^d \to \bbR$ for every $z \in \cZ$, and a distribution $D$ such that, with a bounded initialization $\| w_1 - w^* \| = \bigO{1}$, the output $w_{\gd}$ for GD satisfies
\begin{align*}
    \bbE[F(w_{\gd})] - F(w^*) = \bigOmega{\frac{1}{\eta T} + \frac{1}{n} + \frac{\eta T}{n^2}}.
\end{align*}
Similarly, the output $w_{\sgd}$ for SGD satisfies
\begin{align*}
    \bbE[F(w_{\sgd})] - F(w^*) = \bigOmega{\frac{1}{\eta T} + \frac{1}{n} + \frac{\eta T}{n^2}}.
\end{align*}
\end{theorem}

It is worth noting that we assume bounded initialization $\| w_1 - w^* \| = \bigO{1}$. This is standard and necessary in the generalization literature: the bound will be vacuous and arbitrarily bad if initial point is away from the optimal point with infinite distance.

Similar to the non-realizable setting, the term $\bigOmega{1/(\eta T)}$ reflects the optimization error. In the meanwhile, the term $\bigOmega{1/n}$ comes from a universal hard instance that holds for any deterministic or stochastic gradient methods. The major challenge in lower bound construction is the proof for the term $\bigOmega{\eta T/n^2}$.  Notice that the term $\bigOmega{1/n}$ does not suggest the rest two terms are vacuous since they are hard in the sense of characterizing the relationship between $\eta$, $T$ and $n$. 

Theorem~\ref{thm: lb-1} suggests that known upper bounds are tight not only in sample complexity but also in $T$ and $\eta$ dependence. More specifically, the lower bound for GD tightly matches the upper bound in \cite{nikolakakis2022beyond}, and the lower bound for SGD almost tightly matches the lower bound in \cite{lei2020fine} up to a $\eta$ factor in the second term. Please refer to Table~\ref{tab: summary} for a comparison.

 We will combine the discussion for GD and SGD due to their similarity. Both upper and lower bounds are non-vacuous only when $T = \Theta(n)$ and $\eta = \Theta(1)$: under this configuration, we obtain the optimal sample complexity lower bound $\bigOmega{1/n}$ from Theorem~\ref{thm: lb-1}, which matches the sample complexity upper bound $\bigO{1/n}$ under the regime of $\eta T = \bigO{n}$. We will see in the next subsection that the conclusion is different when $\eta T = \bigOmega{n}$.



\subsection{Realizable Setting:  \texorpdfstring{$\eta T = \bigOmega{n}$}{eta T=Omega(n)}} \label{sec: t_larger_than_n}
In this subsection we focus on the case that allows large or infinite training horizon $\eta T = \bigOmega{n}$. We provide lower bounds for different algorithms and discuss their relationship with upper bounds, as stated in the following theorem.
\begin{theorem} \label{thm: lb-2}
For every $\eta > 0$, $T > 1$, if condition $\eta T = \bigOmega{n}$ holds, then there exists a convex, $1$-smooth and realizable $f(w, z): \bbR^d \to \bbR$ for every $z \in \cZ$, and a distribution $D$ such that, with a bounded initialization $\| w_1 - w^* \| = \bigO{1}$, the output $w_{\gd}$ for GD satisfies
\begin{align*}
    \bbE[F(w_{\gd})] - F(w^*) = \bigOmega{\frac{1}{\eta T} + \frac{1}{n}}.
\end{align*}
Similarly, the output $w_{\sgd}$ for SGD satisfies
\begin{align*}
    \bbE[F(w_{\sgd})] - F(w^*) = \bigOmega{\frac{1}{\eta T} + \frac{1}{n}}.
\end{align*}
\end{theorem}
Here,  GD and SGD again have the same upper and lower bounds. Theorem~\ref{thm: lb-2} indicates that, different from the case $\eta T = \bigO{n}$, our lower bound for both GD and SGD does not match the corresponding upper bounds in \cite{lei2020fine,nikolakakis2022beyond} (see Table~\ref{tab: summary}). For lower bound, we achieve best sample complexity $\Omega(1/n)$ as long as $\eta T \geq n$, whereas for upper bound, the best sample complexity $O(1/n)$ is obtained \emph{only} when we set $\eta T = n$. To conclude, despite the lower bound of sample complexity certifies the optimality of its upper bound, the upper and lower bound still suggest different behavior of generalizarion error: while the upper bound indicates longer training leads to overfit, the lower bound suggests the opposite under the realizable setting.

To obtain the lower bound in Theorem~\ref{thm: lb-2} we employ a strategy similar to the proof of Theorem~\ref{thm: lb-1}. Term $\bigOmega{1/n}$ comes from the universal sample hardness for any algorithm, and term $\bigOmega{1/(\eta T)}$ is obtained from the construction used to prove $\bigOmega{\eta T/n}$ in Theorem~\ref{thm: lb-1}. Albeit the identical construction, the difference between bounds comes from lower bounding $1 - (1-\eta/n)^T$ in two regimes: when $\eta T = \bigO{n}$, we have $1 - (1-\eta/n)^T = \bigOmega{1}$ and when $\eta T = \bigOmega{n}$, we have $1 - (1-\eta/n)^T = \bigOmega{\eta T/n}$. This then leads to a difference in the absolute value of each  coordinates and in the end the difference of overall lower bounds. The details are postponed to Appendix~\ref{appendix: lb-gd-t}.

We conjecture the sample complexity bound under a large or infinite time horizon can be closed by proving upper bound $\bigO{1/n}$ is achievable for GD even when $\eta T$ goes to infinity. We will discuss the conjecture and provide several evidences in Section~\ref{sec: infinite}.

\section{Upper Bounds under \texorpdfstring{$\eta T = \bigOmega{n}$}{\eta T=Omega(n)}} \label{sec: infinite}
In Section~\ref{sec: lb}, we establish lower bounds for both realizable and non-realizable cases. For non-realizable losses, both the upper and lower bounds are tight regardless of the relationship between $T$ and $n$. The result differs for the realizable cases: while upper bound results suffer from large training time, our lower bounds say that overfitting does not happen. It is natural to ask
\begin{center}
    \it Can we close the gap between upper and lower bounds\\ for realizable SCO when $\eta T$ goes to infinity?
\end{center}
We conjecture that the above problem can be tackled by proving GD and SGD can achieve $\bigO{1/n}$ even when $\eta T$ goes to infinity. In the section, we provide evidences to support the conjecture: we consider the examples of one-dimensional function and linear regression. On both examples, $\Theta(1/n)$ sample complexity is achieved for GD and SGD when $\eta T$ is large.
\subsection{One-dimensional Feasibility}
We support our conjecture by providing a first evidence in dimension one: under $d = 1$, we close the gap between upper and lower bound by establishing $\Theta(1/n)$ sample complexity in the rest part of the subsection.

We start by presenting Lemma~\ref{lemma: dim-1-ub-sgd}, which establishes an upper bound for SGD based on the result by \citet{lei2020fine}. 
\begin{lemma} \label{lemma: dim-1-ub-sgd}
In dimension one, if $f(w,z)$ is convex, $1$-smooth and realizable with $z \sim D$, then for every $\eta = \Theta(1)$, there exists $T_0 = \Theta(n)$ such that for $T \geq T_0$, the output $w_{\sgd}$ of SGD satisfies
\begin{equation*}
    \bbE[ F(w_{\sgd})] - F(w^*) = \bigO{\frac{1}{n}}.
\end{equation*} 
\end{lemma}
\begin{proof}
From Theorem~4 in \citet{lei2020fine}, it holds that for realizable cases (we rescale it to $f(w^*, z) = 0$ for each $z$) with step-size $\eta =  \Theta(1)$, it holds that
\begin{equation}
    \bbE[F(w_{\sgd})] = \bigO{\frac{1}{T_0} + \frac{1 + T_0/n}{n}}.
\end{equation}
Therefore, for $T_0= \Theta(n)$, it holds that
\begin{equation*}
    \bbE[F(w_{\sgd})] = \bigO{\nfrac{1}{n}}.
\end{equation*}
For SGD, the iteration formulates the iterate 
\begin{equation*}
    w_{t+1} = w_t - \eta \nabla f(w_t, z_{i_t}),
\end{equation*}
where $z_{i_t}$ is uniformly chosen from $S$. Under the realizable and convex assumption, for any $z_{i_t} \in \cZ$, the iteration becomes
\begin{equation*}
    w_{t+1} - w^* = (1 - \eta \nabla^2 f(\xi, z_{i_t})) (w_t - w^*),
\end{equation*}
using mean value theorem, where $\xi$ is a point between $w_t$ and $w^*$. This indicates that the distance $w_t - w^*$ shrinks in each step for any $z_{i_t} \in \cZ$. Due to the convexity of $F$, it holds that 
$F(w_{t+1}) \leq F(w_t)$. In summary, for any $T \geq T_0$, it holds that 
\begin{equation*}
    \bbE[F(w_{\sgd})] - F(w^*) = \bigO{\nfrac{1}{n}}.
\end{equation*} 
\end{proof}
A similar result can be established for GD, as in the next lemma. Its proof is similar and hence postponed to Appendix~\ref{appendix: infinite}.
\begin{lemma} \label{lemma: dim-1-ub-gd}
In dimension one, if $f(w,z)$ is convex, $1$-smooth and realizable with $z \sim D$, then for every $\eta = \Theta(1)$, there exists $T_0 = \Theta(n)$ such that for $T \geq T_0$, the output $w_{\gd}$ of GD satisfies
\begin{equation*}
    \bbE[ F(w_{\gd})] - F(w^*) = \bigO{\frac{1}{n}}.
\end{equation*} 
\end{lemma}

Unfortunately, we cannot employ the same technique to extend the result to high-dimensional case. However, we show that the gap can be closed for the special case of linear regression in the high-dimensional regime in next subsection.

\subsection{Linear Regression}
In this subsection, we demonstrate that when $\eta T = \bigOmega{n}$, $\Theta(1/n)$ can be achieved on \emph{linear regression} problem, whatever underparameterized ($ d < n$) or overparameterized ($ d \geq n$). 
In realizable (or noiseless) linear regression problems, the $i$-th sample $z_i = (x_i, y_i)$ in dataset $S= \{z_1, \dots, z_n \}$ satisfies that $y_i = x_i^\top w^*$ and $x_i$ is i.i.d. drawn from an unknown distribution. Under the linear predictor $x_i^\top w$, the loss term is defined as $f(w,z) = (y_i - x_i^\top w)^2$. Under this regime, a bounded feature $\|x\| = \bigO{1}$ suffices to guarantee that $f(w, z)$ is convex, $\bigO{1}$-smooth, and realizable. 
In this case, the upper bound would be $\bigO{1/n}$ and the lower bound would be ${\cO}(\log^3n/n)$, which is optimal up to a $\log$-factor. 

We begin by presenting Lemma~\ref{lemma: ub-regression}, which establishes an upper bound using local Rademacher Complexity.
\begin{lemma}[From \citet{srebro2010optimistic}]
\label{lemma: ub-regression}
    In the realizable linear regression cases, for every $\eta > 0$ and $T > 1$, if the feature $x_i$ is bounded, it holds that for the output of SGD
\begin{align*}
    & \bbE[F(w_{\sgd})] - F(w^*) = \bigO{\frac{1}{\eta T} + \frac{\log^3 n}{n} }, 
    \end{align*}
    and also the output for GD
    \begin{align*}
    & \bbE[F(w_{\gd})] - F(w^*) = \bigO{\frac{1}{\eta T} + \frac{\log^3 n}{n}}.
\end{align*}
\end{lemma}
\begin{proof}
    One could directly apply Theorem~1 in \citet{srebro2010optimistic}.
    Specifically, we plug in the realizability assumption and the Rademacher complexity of linear function class, which is in order $\bigO{1/\sqrt{n}}$ in bounded norm cases. 
\end{proof}
Lemma~\ref{lemma: ub-regression} 
established a sample complexity rate of $\Theta(1/n)$ for linear regression when $T$ grows large. Our evidence on both dimension one case and regression suggests the gap in the regime $\eta T = \bigOmega{n}$ might be closed by improving the upper bounds of excess risk or sample complexity. However, the approach do not generalize to general convex functions due to as convex functions have much larger Radamacher complexity. We hope our analysis can motivate future exploration into the topic.
\section{Proof Overviews} \label{sec: proof}
In this section we provide a brief overview regarding our technique used in the proofs of theorems for realizable cases in Section~\ref{sec: lb}. In particular, we will focus on the lower bound construction for the output of GD when $\eta T = \bigO{n}$, i.e. first part in Theorem~\ref{thm: lb-1}, to showcase the major intuition and idea behind our constructions. 

As discussed in the above section, the main technical difficulty in the GD part of Theorem~\ref{thm: lb-1} lies in proving
\begin{equation} \label{eq: lb-major}
    \bbE[F(w_{\gd})] - F(w^*) = \bigOmega{\frac{\eta T}{n^2}}.
\end{equation}
The proof of rest terms is based on easier constructions and we recommend referring to Lemma~\ref{lemma: lb-sample} and Lemma~\ref{lemma: lb-suboptimality} in Appendix~\ref{appendix: minor}. These lemmas are general and hold for any deterministic or stochastic gradient methods. Here we focus on the proof of \eqref{eq: lb-major}. Our technique is novel and inspired by the work of \citet{amir2021sgd, sekhari2021sgd}. However, their construction critically relies on nonsmoothness and nonrealizability. 

We start by considering running GD on the following 2-dimensional quadratic function
\begin{equation*}
    h(x, y) = \frac{\alpha x^2}{2} + \frac{y^2}{2} - 2\sqrt{\alpha} xy = \frac{1}{2} \left| \sqrt{\alpha} x -  y \right|^2
\end{equation*}
with step-size $\eta$ and initialization $x_1 = 1$ and $y_1 = 0$. We choose a small enough $\alpha = \nfrac{1}{\eta T} \ll 1$. In every iteration, since $\alpha$ is small, $x$ is pulled back to zero slowly: it is easy to lower bound the value since 
\begin{equation*}
    x_{t+1} \geq (1 - \alpha\eta) x_{t} \geq e^{-\alpha \eta t} x_1  \geq e^{-t/T} x_1 \geq 1/e = \Theta(1).
\end{equation*}
Hence $x_t = \bigOmega{1}$ for any $t \in [T]$. Meanwhile, 
coordinate $y$ is simultaneously (1) pushed away from zero by $x$ on the scale of $\bigOmega{\eta\sqrt{\alpha}}$ and (2) pulled back towards zero by itself. As a result, despite the pulling influence, we can still guarantee that $y_t$ is bounded away from zero for all $t \in [T]$.

We now want to improve over the naive two-dimensional quadratic example to make sure that multiple coordinates are bounded away from zero. This intuitively might provide a hard instance for the GD algorithm. Also, we hope stochasticity plays a role in the hard instance such that we can introduce the factor of $n$.  We then devise the following instance $g(w, z): \bbR^{n+1} \times \cZ \to \bbR$ belonging to the realizable smooth SCO setting:  
\begin{equation}
    g(w, z = i) = \frac{\alpha}{2} x^2 + \frac{1}{2} \big(y(i)\big)^2  - \sqrt{\alpha} x \cdot y(i) = \frac{1}{2} \left|\sqrt{\alpha}x - y(i)\right|^2
\end{equation}
where $w = (x,y)$, $x \in \bbR$, $y \in \bbR^n$ and $z \sim \text{Unif}([n])$.  We still set parameter $\alpha$ to be $1/(\eta T)$. We are given a dataset $S$ of $n$ examples i.i.d. from the distribution. This leads to population loss
\begin{equation}
    G(w) = \bbE_{z\sim\text{Unif}([n])}[g(w,z)] =  \frac{1}{2n} \left\| y - \sqrt{\alpha} x \right\|^2.
\end{equation}

We generalize the idea from the two-dimensional case to $n+1$ dimension. To this end, we need every example $z_i \in S$ corresponds to one coordinate $y(i)$. This is, however, an improbable event that occurs with probability $\Theta(\sqrt{n} \cdot e^{-n})$. We use the intuition from \citet{amir2021sgd,sekhari2021sgd}: if we consider multiple independent copies of $g(w,z)$, then with probability $\Theta(1)$, there exists at least one copy that satisfies the condition.

We focus on the particular copy only. Our calculations shows that under assumption $\eta T = \bigO{n}$, it holds that, for any $t \in [T]$, (1) $x_t = \Theta(1)$ and (2) $y_t(i) = \bigOmega{\sqrt{\eta t}/n}$ for any coordinates $i \in [n]$. With a slight abuse of the notation $w$, we put everything together and guarantee that 
\begin{align*}
    F(w_{\gd}) - F(w^*) = \bigOmega{\frac{1}{2n} \cdot \| y_T \|^2} = \bigOmega{\frac{\eta T}{n^2}}.
\end{align*}

The details in the proof of Theorem~\ref{thm: lb-1} can be found in Appendix~\ref{appendix: lb-gd-tn}. The idea behind the proof for the case $\eta T = \bigOmega{n}$ (Theorem~\ref{thm: lb-2}) differs only in calculations and hence we omit the repetition. Proof for SGD is also similar. The details of proof for other theorems can be found in Appendix~\ref{appendix: t_equal_n}.
\section{Additional Related Work}
Generalization in stochastic convex optimization has been extensively explored in the literature~\citep{boyd2004convex,shalev2014understanding}, with one-pass SGD~\citep{pillaud2018exponential}, multi-pass SGD~\citep{pillaud2018statistical,sekhari2021sgd,lei2021generalization}, DP-SGD~\citep{bassily2019private,ma2022dimension}, ERM solution~\citep{feldman2016generalization,aubin2020generalization} and so forth. One of the most famous results is that one-pass SGD can achieve an optimal error rate of $\cO(1/\sqrt{n})$ in convex optimization, even in the presence of non-smooth loss functions~\citep{nemirovskij1983problem}.

However, for realizable problems, existing analyses typically focus only on upper bounds ~\citep{lei2020fine,nikolakakis2022beyond,schliserman2022stability,taheri2023generalization} and corresponding lower bounds are lacked. Realizability is closely related to label noise, which can have a substantial impact on generalization performance~\citep{song2019does,harutyunyan2020improving,DBLP:conf/iclr/TengMY22,wen2022realistic}. 

For lower bounds,
\cite{amir2021sgd} show that no less than $\bigOmega{1/\epsilon^4}$ steps is needed for GD to achieve $\epsilon$-excess risk, whereas SGD needs only $\bigO{1/\epsilon^2}$. This is an iteration bound, whereas some other works (including our work) focus on the sample complexity bound: \cite{sekhari2021sgd} further indicate that GD suffers from a $\bigOmega{1/n^{5/12}}$ sample complexity, which is slower than the well-established bound $\Theta(1/\sqrt{n})$ for SGD \citep{nemirovskij1983problem}.
Besides the upper/lower bound mentioned above, a line of lower bounds in generalization analysis typically focuses on the failure of techniques. 
For instance, despite the optimal rate of $\cO(1/\sqrt{n})$ in convex optimization, uniform convergence only returns a lower bound of $\Omega(\sqrt{d/n})$~\citep{shalev2010learnability,feldman2016generalization}, leading to a constant lower bound in overparameterized regimes. 
A line of works further illustrate the inherent weakness of uniform convergence~\citep{nagarajan2019uniform,glasgow2022max}.
Regarding stability-based bounds, \citet{bassily2020stability} presents a lower bound under non-smooth convex losses.

To bridge the gap between lower and upper bound, a fast rate upper bound in order $O(1/n)$ is required. 
One of the most well-known fast-rate bound is local Rademacher complexity, which works well under low-noise regimes~\citep{bartlett2005local}. 
However, it typically relies on a specific function class and may not be directly applied into the general convex optimization regimes~\citep{steinwart2007fast,srebro2010optimistic,zhou2021optimistic}. 
Alternatively, stability-based analyses have shown promise and work well in convex optimization regimes, which have the potential to provide fast-rate generelization bound~\citep{bousquet2002stability,hardt2016train,feldman2019high,zhang2022stability}.
In addition to these bounds, one can also derive fast rate bound for finite-dimensional cases~\citep{lee1996importance,bousquet2002concentration}, aggregation~\citep{tsybakov2004optimal,chesneau2009adapting,dalalyan2018exponentially}, PAC-Bayesian and information-based analysis~\citep{yang2019fast,grunwald2021pac}.

\section{Conclusion}
In this work, we focus on generalization bounds under the smooth SCO setting. In particular, we provide lower bounds for excess risk as a function sample size $n$, the learning rate $\eta$ and the iteration $T$ under three settings: (1) non-realizable, (2) realizable with $\eta T = \bigO{n}$, and (3) realizable with $\eta T =  \bigOmega{n}$.  For the first two cases, our lower bounds match the corresponding upper bounds and certificate the optimal sample complexity. Nevertheless, under the realizable case with $\eta T = \bigO{n}$, we observe a gap between extisting upper bounds and lower bounds. We conjecture that this gap can be closed by improving the upper bound under the long time horizon regime, and provide evidence in the one-dimensional problem and the linear regression problem to support our hypothesis.






\newpage

\appendix

\section{Missing Proofs from Section~\ref{sec: nonrealizable}} \label{appendix: nonrealizable}
\subsection{Proof of Theorem~\ref{thm: lb-sco-gd}} \label{appendix: lb-sco-gd}
The theorem provides an excess risk lower bound $\bigOmega{\nfrac{1}{\eta T} + \nfrac{\eta T}{n}}$ for GD under the \emph{non-realizable} smooth SCO scenario. The result is obtained by combining a $\bigOmega{1/{\eta T}}$ bound in Lemma~\ref{lemma: lb-suboptimality} and a $\bigOmega{\eta T/n}$ bound in Lemma~\ref{lemma: lb-gd-overfit} stated below. The first bound reflects an optimization error and is postponed to Appendix~\ref{appendix: lb-suboptimality}. In the rest part, we present the proof of the latter lemma.
\begin{lemma} \label{lemma: lb-gd-overfit}
    For any $\eta > 0$, $T > 1$, there exists a convex, $1$-smooth $f(w,z): \bbR \to \bbR$ for every $z \in \cZ$, and a distribution $D$ such that, with probability $\Theta(1)$, the output $w_{\gd}$ for GD satisfies
    \begin{align*}
        F(w_\gd) - F(w^*) = \bigOmega{\frac{\eta T}{n}}.
    \end{align*}
\end{lemma}

\begin{proof}
We define loss function $f: \bbR \times \cZ \to \bbR$ as
\begin{align*}
    f(w, z) = \frac{w^2}{2\eta T}  + zw
\end{align*}
where $z \sim \text{Unif}(\{ \pm 1\})$. It is obvious that $f(w,z)$ is $1$-smooth and convex since $\eta T \geq 1$. The population risk is computed as
\begin{align*}
    F(w) = \bbE_{z\sim\text{Unif}(\{\pm 1\})}[f(w, z)] = \frac{w^2}{2\eta T}.
\end{align*}
The minimizer is then $w^* = 0$. GD formulates the following recurrence on dataset $S$ with initialization $w_1=0$:
\begin{align*}
    w_{t+1} = w_t - \frac{\eta}{n} \sum_{i=1}^n\left( \frac{w_t}{\eta T} + z_i \right) = \left(1-\frac{1}{T} \right) w_t - \frac{\eta}{n}\sum_{i=1}^nz_i,
\end{align*}
where each $z_i \sim \text{Unif}(\{\pm 1\})$ for $i \in [n]$. We want to use an anti-concentration result to lower bound the recurrence: from Lemma 7 in \citet{sekhari2021sgd}, with probability $\bigOmega{1}$, it holds that
\begin{align*}
    \sum_{i=1}^n z_i \leq - \frac{\sqrt{n}}{2}.
\end{align*}
We get lower bound
\begin{align*}
    w_{t+1} \geq  \left(1-\frac{1}{T} \right) w_t + \frac{\eta}{2\sqrt{n}}.
\end{align*}
Then we have for any $t \in [T]$
\begin{align*}
    w_t & \geq \frac{\eta}{2\sqrt{n}} \left( 1 + \left( 1 - \frac{1}{T}\right) + \cdots + \left( 1 - \frac{1}{T}\right)^{t-1} \right) \\
    & \geq \frac{\eta t}{8\sqrt{n}}
\end{align*}
where the second inequality is due to the fact 
\begin{align*}
    1 > 1 - \frac{1}{T} > \cdots  \left( 1 - \frac{1}{T}\right)^t > \cdots > \left( 1 - \frac{1}{T}\right)^{T} \geq \frac{1}{4}
\end{align*}
for any $t \in [T]$ and $T \geq 2$. Then the average is lower bounded as
\begin{align*}
    \bar{w}_T = \frac{1}{T} \sum_{t=1}^T w_t \geq \sum_{t=1}^T \frac{\eta t}{8\sqrt{n}} = \frac{\eta (T-1)}{16\sqrt{n}}. 
\end{align*}
As a result, we have
\begin{align*}
    F(w_{\gd}) - F(w^*) = \frac{w_{\gd}^2}{2\eta T}  = \bigOmega{\frac{\eta T}{n}},
\end{align*}
which is the desired result.
\end{proof}

\subsection{Proof of Theorem~\ref{thm: lb-sco-sgd}}
Similar to the proof of Theorem~\ref{thm: lb-sco-gd}, we prove the excess risk lower bound for SGD by combining Lemma~\ref{lemma: lb-suboptimality} and the following lemma.
\begin{lemma} \label{lemma: lb-sgd-overfit}
    For any $\eta > 0$, $T > 1$,  there exists a convex, $1$-smooth $f(w,z): \bbR \to \bbR$ for every $z \in \cZ$, and a distribution $D$ such that, with probability $\Theta(1)$, the output $w_{\sgd}$ for SGD satisfies
    \begin{align*}
        \bbE[F(w_{\sgd})] - F(w^*) = \bigOmega{\frac{\eta T}{n}}.
    \end{align*}
\end{lemma}

\begin{proof}
We use the same construction in Lemma~\ref{lemma: lb-gd-overfit}. Consider dataset $S=\{z_1,\dots,z_n\}$ where $z_i \sim \text{Bern}(\{\pm 1\})$. Given $x_t$, SGD formulates the following recurrence on dataset $S$ with initialization $w_1=0$:
\begin{align*}
    \bbE[w_{t+1}] = w_t - \frac{\eta}{n} \sum_{i=1}^n\left( \frac{w_t}{\eta T} + z_i \right) = \left(1-\frac{1}{T} \right) w_t - \frac{\eta}{n}\sum_{i=1}^nz_i,
\end{align*}
where $z_i \sim \text{Unif}(\{\pm 1\})$. From Lemma 7 in \citet{sekhari2021sgd}, with probability $\bigOmega{1}$, it holds that
\begin{align*}
    \sum_{i=1}^n z_i \leq - \frac{\sqrt{n}}{2}.
\end{align*}
Then we get lower bound
\begin{align*}
    \bbE[w_{t+1}] \geq  \left(1-\frac{1}{T} \right) w_t + \frac{\eta}{2\sqrt{n}}.
\end{align*}
Similar to the proof of Lemma~\ref{lemma: lb-gd-overfit}, we have for any $t \in [T]$
\begin{align*}
    \bbE[w_t] & \geq \frac{\eta}{2\sqrt{n}} \left( 1 + \left( 1 - \frac{1}{T}\right) + \cdots + \left( 1 - \frac{1}{T}\right)^{t-1} \right) \\
    & \geq  \frac{\eta t}{8\sqrt{n}}.
\end{align*}
Then the average is lower bounded as
\begin{align*}
    \bbE[\bar{w}_T] = \frac{1}{T} \sum_{t=1}^T \bbE[w_t] \geq \sum_{t=1}^T \frac{\eta t}{8\sqrt{n}} = \frac{\eta (T-1)}{16\sqrt{n}}. 
\end{align*}
As a result, we have
\begin{align*}
    \bbE[F(w_{\sgd})] - F(w^*) \geq F(\bbE[w_{\sgd}]) - F(w^*) = \frac{(\bbE[w_{\sgd}])^2}{2\eta T}  = \bigOmega{\frac{\eta T}{n}},
\end{align*}
by Jensen's inequality.
\end{proof}

\section{Missing Proofs from Section~\ref{sec: t_equal_n} and Section~\ref{sec: t_larger_than_n} } \label{appendix: t_equal_n}

\subsection{Proof of Theorem~\ref{thm: lb-1}} \label{appendix: lb-gd-tn}
The proof of GD is immediate from combining lower bound $\bigOmega{1/{\eta T}}$ in Lemma~\ref{lemma: lb-suboptimality}, lower bound $\bigOmega{1/n}$ in Lemma~\ref{lemma: lb-sample}, and most importantly, lower bound $\bigOmega{\eta T/n^2}$ in Lemma~\ref{lemma: lb-gd-tn} to be stated below. In precise, Lemma~\ref{lemma: lb-gd-tn} is the core part of our result and gives a lower bound of $\bigOmega{\eta T/n}$ when $\eta T = \bigO{n}$, and a lower bound of $\bigOmega{1/\eta T}$ when $\eta T = \bigOmega{n}$ for GD. The latter part is used in the proof of Theorem~\ref{thm: lb-2}. We postpone its proof to Appendix~\ref{appendix: lb-gd}. The proof of the rest two lemmas can be found in Appendix~\ref{appendix: lb-sample}, \ref{appendix: lb-suboptimality}.
\begin{lemma} \label{lemma: lb-gd-tn}
For every $\eta > 0$, $T > 1$, if $\eta T =  \bigO{n}$, then there exists a convex, $1$-smooth and realizable $f(W, Z): \bbR^{(n+1) \times m} \times \cZ^m \to \bbR$ for every $z \in \cZ$, and a distribution $D$ such that, with initialization $\| W_1 - W^* \| = \bigO{1}$, the output $W_{\gd}$ for GD satisfies
\begin{align*}
 \bbE[F(W_{\gd})] - F(W^*) = \bigOmega{\frac{\eta T}{n^2}}.  
\end{align*}
In specific, $m$ is an integer with $m = \Theta(e^n/\sqrt{n})$. Similarly, if $\eta T =  \bigOmega{n}$, then 
it satisfies
\begin{align*}
 \bbE[F(W_{\gd})] - F(W^*) = \bigOmega{\frac{1}{\eta T}}.  
\end{align*}
\end{lemma}
Similar to the proof of GD, the result on SGD is also obtained by combining lower bound constructions in the following lemma and Lemma~\ref{lemma: lb-sample}, \ref{lemma: lb-suboptimality} in Appendix~\ref{appendix: lb-sample}, \ref{appendix: lb-suboptimality}. Lemma~\ref{lemma: lb-sgd-tn} establishes a lower bound of $\bigOmega{\eta T/n}$ when $\eta T = \bigO{n}$, and a lower bound of $\bigOmega{1/\eta T}$ when $\eta T = \bigOmega{n}$ for SGD. Its proof can be found in Appendix~\ref{appendix: lb-sgd}.
\begin{lemma} \label{lemma: lb-sgd-tn}
For every $\eta > 0$, $T > 1$, if $\eta T =  \bigO{n}$, then there exists a convex, $1$-smooth and realizable $f(W, Z): \bbR^{(n+1) \times m} \times \cZ^m \to \bbR$ for every $Z \in \cZ^m$, and a distribution $D$ such that, with initialization $\| W_1 - W^* \| = \bigO{1}$, the output $W{\sgd}$ for SGD satisfies
\begin{align*}
 \bbE[F(W_{\sgd})] - F(W^*) = \bigOmega{\frac{\eta T}{n^2}}.  
\end{align*}
In specific, $m$ is an integer with $m = \Theta(e^n/\sqrt{n})$. Similarly, if $\eta T =  \bigOmega{n}$, then 
it satisfies
\begin{align*}
 \bbE[F(W_{\sgd})] - F(W^*) = \bigOmega{\frac{1}{\eta T}}.  
\end{align*}
\end{lemma}

\subsection{Proof of Theorem~\ref{thm: lb-2}} \label{appendix: lb-gd-t}
Similar to the proof of Theorem~\ref{thm: lb-1}, the proof of GD is obtained from combining the lower bounds in Lemma~\ref{lemma: lb-sample}, Lemma~\ref{lemma: lb-suboptimality}, and Lemma~\ref{lemma: lb-gd-tn}. In particular, Lemma~\ref{lemma: lb-gd-tn} gives a lower bound of $\bigOmega{\eta T/n}$ when $\eta T = \bigOmega{n}$. 

Concurrently, the proof of SGD is obtained from combining the lower bounds in Lemma~\ref{lemma: lb-sample}, Lemma~\ref{lemma: lb-suboptimality}, and Lemma~\ref{lemma: lb-sgd-tn}. In particular, Lemma~\ref{lemma: lb-sgd-tn} gives a lower bound of $\bigOmega{1/{\eta T}}$ when $\eta T = \bigOmega{n}$. 

\subsection{Proof of Lemma~\ref{lemma: lb-gd-tn}} 

This subsection contains the proof of Lemma~\ref{lemma: lb-gd-tn}, along with another two supportive lemmas. We first present the proof of the major lemma. 
\label{appendix: lb-gd}

\begin{proof} We construct the following instance to obtain the lower bound, where $f: \bbR^{(n+1) \times m} \times \cZ^m \to \bbR$ is
    \begin{align} \label{eq: lb-instance1}
        & f(W,Z) = \sum_{j=1}^m g(w^{(i)}, z^{(i)}) 
    \end{align}
    with a positive integer $m$, and $g$ defined as
    \begin{align} \label{eq: lb-instance2}
        & g(w, z = i) = \frac{\alpha}{2} x^2 + \frac{1}{2} \big(y(i)\big)^2  - \sqrt{\alpha} x \cdot y(i) = \frac{1}{2} \Big(\sqrt{\alpha}x - y(i)\Big)^2,
    \end{align}
where $W = ( w^{(1)}, \cdots, w^{(m)})$ is a large vector formed by concatenating by $m$ vectors $\{w^{(j)} \ |\ w^{(j)} \in \bbR^{n+1}\}_{j=1}^m$, and $Z = ( z^{(1)}, \cdots, z^{(m)})$ denotes a large sample concatenated by $m$ copies of independent samples $\{ z^{(j)}\ |\ \ z^{(j)} \in  [n]\}$. We will omit the upscript of $j$ when it does not lead to confusion. Each $w$ is split into $w = (x, y)$ with $x \in \bbR$ and $y \in \bbR^n$ in the function $g$, and we define $y(i)$ as the $i$-th coordinate of $y$. We assume $z \sim \text{Unif}([n])$ i.i.d., and set parameters $m = \Theta(e^n/\sqrt{n})$, $\alpha = C/(\eta T)$, where $C \leq 1$ is a constant. 
Intuitively, $f$ can be regarded as the summation over $m$ copies of $g(w^{(j)},z^{(j)})$. 
Such a construction $f$ satisfies the conditions in the statement of this lemma (see Lemma~\ref{lemma: supp-1} below).

Lemma~\ref{lemma: supp-2} (also see below) shows that: with constant probability, there exists at least one copy of $\{z^{(j)}_i\}_{i\in[n]}$ (for clarification, $z^{(j)}_i$ is the $j$-th component in the $i$-th sample $Z_i$ within the dataset $S=\{Z_1,\dots,Z_n\}$) satisfying
\begin{align*}
    z^{(j)}_i = i, \quad \text{for all} \quad i \in [n],
\end{align*}
without the loss of generality, we consider the identity permutation $\vpi(i) = i$. We use the following initialization:
\begin{align*}
    x_1^{(k)} = \begin{cases} 1, \qquad k = j,\\ 0, \qquad k \ne j;\end{cases} \quad\text{and}\qquad y_1^{(k)} = 0, \qquad \forall k \in [m].
\end{align*}
We have then $\| W_1 - W^*\| = \bigO{1}$. This allows us to focus on the $j$-th component only and hence we suppress the upscripts. In this context, the stochastic loss function $g$ on this copy is written as
    \begin{equation}
        g(w, z_i) = \frac{\alpha}{2} (x )^2 + \frac{1}{2} \|y \|^2 - \frac{x \sqrt{\alpha}}{n} y (i), \qquad \forall i \in [n].
    \end{equation}
From the above construction, GD formulates the following update
\begin{align*}
        w_{t+1} = w_t - \frac{\eta}{n} \sum_{i=1}^n \nabla_w g(w_t, z_i)
    \end{align*}
    with initialization $x_1 = 1$, $y_1 = 0$.
The stochastic gradient is computed as 
    \begin{align}
        \nabla_x g(w, z_i) = \alpha x - \sqrt{\alpha}y(i), \qquad \nabla_{y} g(w, z_i) = (y(i) - \sqrt{\alpha}x) \cdot \ve_i.
    \end{align}

    Since all coordinates in $y$ are equivalent in the construction, we suppress the index of $i$ and write $y_t = y_t(i)$ for any $i \in [n]$, $t \in [T]$. Then it formulates
    \begin{align*}
        & x_{t+1} = x_t - \eta \alpha x_t + \frac{\eta \sqrt{\alpha}}{n}\sum_{i=1}^ny_t(i) = (1-\alpha \eta)x_t + \eta \sqrt{\alpha}y_t, \\
        & y_{t+1} = y_t - \frac{\eta}{n}y_t + \frac{\eta\sqrt{\alpha}}{n}x_t = \left(1 - \frac{\eta}{n} \right)y_t + \frac{\eta\sqrt{\alpha}}{n}x_t.
    \end{align*}

    We next provide both upper and lower bounds for $x_t$ and $y_t$. 
    We give an upper bound for $x_t$ and $y_t$ by the following induction. If condition
    \begin{equation} \label{eq: induction-GD-2}
        x_{t} \leq 1, \qquad y_t \leq \sqrt{\alpha}
    \end{equation} 
    holds for $t$, then the above condition also holds for $t+1$:
    \begin{align*}
        x_{t+1} & \leq (1 - \alpha \eta) + \eta\sqrt{\alpha} \cdot \sqrt{\alpha} = 1 - \frac{\eta}{\eta T} + \frac{\eta }{\eta T}  \leq 1, \\
        y_{t+1} & \leq \left(1 - \frac{\eta}{n} \right) \sqrt{\alpha} + \eta \frac{\sqrt{\alpha}}{n} = \sqrt{\alpha}.
    \end{align*}
    Then by induction we conclude that \eqref{eq: induction-GD-2} is true.
    For any $t \in [T]$ with $T \geq 2$, the lower bound for $x_t$ is much simpler to compute under our choice of parameter $\alpha = C/(\eta T)$:
    \begin{align*}
        x_{t+1} & \geq (1 - \alpha \eta)  x_{t} \geq (1 - \alpha \eta )^tx_1 = 4^{-Ct/T} \geq 4^{-C}.
    \end{align*}
    Hence $\bar{x}_T = \frac{1}{T} \sum_{t=1}^T x_t = \Theta(1)$. This then allows us to lower bound $y$ at iteration $t \in [T]$:
    \begin{align*}
        y_t & \geq \left(1-\frac{\eta}{n}\right) y_{t-1} + \frac{\eta\sqrt{\alpha}}{4^Cn} \\
        & \geq \frac{\eta\sqrt{\alpha}}{4^Cn} \cdot \left( 1 + (1-\eta/n) + \cdots (1-\eta/n)^{t-1} \right) \\
        & \geq \frac{\eta\sqrt{\alpha}}{4^Cn} \cdot \frac{1 - (1-\eta/n)^t}{1 - (1 - \eta/n)} .
     \end{align*}
    Now, we discuss two cases: $\eta T = \bigO{n}$ and $\eta T = \bigOmega{n}$.
    \paragraph{Case $\eta T = \bigO{n}$.} We decompose $t = n \cdot \tfrac{t}{n}$ and obtain
    \begin{align*}
        y_t & \geq \frac{\eta\sqrt{\alpha}}{4^Cn} \cdot \frac{1 - (1-\eta/n)^t}{1 - (1 - \eta/n)} = \frac{\eta\sqrt{\alpha}}{4^Cn} \cdot \frac{1 - (1-\eta/n)^{\tfrac{t}{n} \cdot n}}{1 - (1 - \eta/n)} \\
        & \overset{\text{(A)}}{\geq} \frac{\eta \sqrt{\alpha}}{4^C}\left( \frac{t}{n} - \frac{\eta t^2}{2n^2}\right)  \overset{\text{(B)}}{=} \frac{\eta t \sqrt{\alpha}}{2\cdot 4^Cn} = \sqrt{\frac{\eta}{CT}} \cdot \frac{t}{2\cdot 4^Cn}
    \end{align*}
    where $\text{(A)}$ is due to Taylor expansion, $\text{(B)}$ is due to the condition $\eta t \leq \eta T = \bigO{n}$ and $\alpha = C/(\eta T)$. We then calculate the average output
    \begin{align*}
        \bar{y}_T = \frac{1}{T} \sum_{t=1}^T y_t = \frac{1}{T} \sum_{t=1}^T \sqrt{\frac{\eta}{CT}} \cdot \frac{t}{2\cdot 4^Cn} \geq \frac{1}{4\cdot4^Cn} \cdot \sqrt{\frac{\eta T}{C}}.
    \end{align*}
    We return to the original $f(w,z)$ by inserting the above analysis on the $j$-th component:
    \begin{align*}
        \bbE[F(W_{\gd})] \geq \bbE[G(w^{(j)}_{\gd})] \geq \frac{1}{n} \sum_{i=1}^n \left(\frac{x^{(j)}_{\gd}}{\sqrt{\eta T}} - y^{(j)}_{\gd}(i)\right)^2 \geq \bigOmega{\max\left\{ \frac{1}{\eta T,} \ \frac{\eta T}{n^2 }\right\}} = \bigOmega{\frac{\eta T}{n^2 }}
    \end{align*}
    where the last inequality is due to the fact $4^{-C} \leq x_{\gd} \leq 1$. We can always choose a proper $C$ such that the difference is non-vanishing.
    \paragraph{Case $\eta T = \bigOmega{n}$.} We can directly lower bound $y_t$ as
    \begin{align*}
        y_t \geq \frac{\eta \sqrt{\alpha}}{4^C n} \cdot \frac{n}{2\eta} = \frac{1}{2\cdot4^C\sqrt{C \eta T}} = \bigOmega{\frac{1}{\sqrt{\eta T}}}
    \end{align*}
    since $(1 - \eta/n)^t \leq 1/2$ when $\eta T = \bigOmega{n}$. We then calculate the average output
    \begin{align*}
        \bar{y}_T = \frac{1}{T} \sum_{t=1}^T y_t = \frac{1}{T} \sum_{t=1}^T \bigOmega{\frac{1}{\sqrt{\eta T}}} \geq \bigOmega{\frac{1}{\sqrt{\eta T}}}.
    \end{align*}
    Similarly, we return to the original $f(w,z)$ by inserting the above analysis on the $j$-th component and obtain a non-vanishing lower bound by choosing proper $C$:
    \begin{align*}
        \bbE[F(W_{\gd})] \geq \bbE[G(w^{(j)}_{\gd})] \geq \frac{1}{n} \sum_{i=1}^n \left(\frac{x^{(j)}_{\gd}}{\sqrt{\eta T}} - y^{(j)}_{\gd}(i)\right)^2 \geq \bigOmega{\frac{1}{\eta T}}.
    \end{align*}
    This completes our proof.
\end{proof}
We proceed to prove the supporting lemmas.
\begin{lemma}
    \label{lemma: supp-1}
    Suppose $\alpha = \Theta(1/\eta T)$, then $f$ is $1$-smooth, convex and realizable over $D$.
\end{lemma}
\begin{proof}
    When condition $\eta T \geq 1$ holds, it is easy to check that $g(w,z)$ is $1$-smooth and convex for any $w \in \bbR^{n+1}$ and $z \in [n]$. The population risk $G$ is 
    \begin{align*}
        G(w) = \bbE_{z\sim\text{Unif}([n])}[g(w,z)] = \frac{\alpha}{2} x^2 + \frac{1}{2n} \|y\|^2 - \frac{\sqrt{\alpha}}{n}x \cdot \vone^\top y = \frac{1}{2n} \left\| y - \sqrt{\alpha} x \cdot \vone \right\|^2,
    \end{align*}
    which attains minimum at $(x^*, y^*) = (0, 0)$. So $g(w,z)$ satisfies the realizable condition.  It is easy to conclude $f$ is also $1$-smooth, convex and realizable.
\end{proof}

\begin{lemma}
\label{lemma: supp-2}
    Consider dataset $S = \{ Z_1, \dots, Z_n\}$ defined in Lemma~\ref{lemma: lb-gd-tn}. Suppose $m$ is a positive integer satisfying $ m = \Theta(e^n/\sqrt{n})$, then there exists at least one component $z^{(j)}, j \in [m]$ such that
    \begin{align*}
        z^{(j)}_i = \vpi(i), \quad \text{for all} \quad i \in [n]
    \end{align*}  
    where $\vpi:[n] \to [n]$ is any permutation on $[n]$.
\end{lemma}
\begin{proof}
    We define the following probability event: given dataset $S = \{ Z_1, \dots, Z_n\}$, we focus on the $j$-th component $\{z^{(j)}_1, \dots ,z^{(j)}_n\}$ and define event $\cE_j$ as 
    \begin{equation*}
        \cE_j = \left\{ z^{(j)}_i = \vpi(i) \text{ for any } i \in [n] \right\}
    \end{equation*}
    where $\vpi:[n] \to [n]$ is any fixed permutation on $[n]$. Intuitively, when $\cE_j$ happens, each coordinates of $y^{(j)}$ is selected only for once in dataset $S$. For any fixed $j \in [m]$, the probability of $\cE_j$ happens is calculated from the without-replacement sampling: 
    \begin{align*}
        p := \pr[\text{Event } \cE_j \text{ happens}] = 1 \cdot \frac{n-1}{n} \cdots \cdot  \frac{1}{n} = \frac{n!}{n^n} = \Theta(\sqrt{n}\cdot e^{-n})
    \end{align*}
    where the last step is from Stirling approximation $\sqrt{2\pi n } \left(\frac{n}{e}\right)^n e^{\tfrac{1}{12n+1}} < n! < \sqrt{2\pi n } \left(\frac{n}{e}\right)^n e^{\tfrac{1}{12n}}$ for any $n \geq 1$. ensures event $\cE_j$ to happen.

    We now prove that with $\bigOmega{1}$ probability, there exists at least a $j \in [m]$ such that $\cE_j$ happens using the second moment method. Denote $R$ to be the random variable counting the number of $\{\cE_j\}_{j \in [m]}$ happens. Using second moment method, we upper bound the following probability: 
    \begin{equation}
        \pr[R > 0] \geq \frac{(\bbE[R])^2}{\bbE[R^2]} = \frac{m^2p^2}{mp(1-p)} \geq \frac{mp}{2} = \frac{1}{2}.
    \end{equation}
    So with probability $\bigOmega{\frac{1}{2}}$, we have at least one copy fulfilling the statement.
\end{proof}

\subsection{Proof of Lemma~\ref{lemma: lb-sgd-tn}} \label{appendix: lb-sgd}
\begin{proof}
    We utilize the similar strategy employed in Lemma~\ref{lemma: lb-gd-tn} and consider the same $f(W,Z)$ defined in Eq.~\eqref{eq: lb-instance1}, Eq.\eqref{eq: lb-instance2}. Lemma~\ref{lemma: supp-2} (see below) shows that: with constant probability, there exists at least one copy of $\{z^{(j)}_i\}_{i\in[n]}$ satisfying (without the loss of generality, we consider the identical permutation $\vpi(i) = i$)
    \begin{align*}
        z^{(j)}_i = i, \quad \text{for all} \quad i \in [n].
    \end{align*}
    We use the following initialization:
    \begin{align*}
        x_1^{(k)} = \begin{cases} 1, \qquad k = j,\\ 0, \qquad k \ne j;\end{cases} \quad\text{and}\qquad y_1^{(k)} = 0, \qquad \forall k \in [m].
    \end{align*}
    We have then $\| W_1 - W^*\| = \bigO{1}$. This allows us to focus on the $j$-th component only and hence we suppress the upscripts. In this context, the stochastic loss function $g$ on this copy is written as
    \begin{equation}
        g(w, z_i) = \frac{\alpha}{2} (x )^2 + \frac{1}{2} \|y \|^2 - \frac{x \sqrt{\alpha}}{n} y (i), \qquad \forall i \in [n].
    \end{equation}
    SGD formulates the update:
    \begin{align*}
        w_{t+1} = w_{t+1} - \eta g(w_t, z_{i_t}) 
    \end{align*}
    where $z_{i_t} \sim \text{Unif}([n])$. The initialization is $x_1 = 1$, $y_1 = 0$. Based on the current value of $w_t$, under expectation, we have
    \begin{align*}
        \bbE[w_{t+1}] = w_t - \frac{\eta}{n} \sum_{i=1}^n \nabla_w g(w_t, z_i).
    \end{align*}
    We write the update of $\bbE[x_t]$ and $\bbE[y_t]$ by plugging stochastic gradients: it easy to see all coordinates in $\bbE[y_t]$ are equivalent, we suppress the index of $i$ and write $y_t = y_t(i)$ for any $i \in [n]$, $t \in [T]$. Then it formulates
    \begin{align*}
        & \bbE[x_{t+1}] = x_t - \eta \alpha x_t + \frac{\eta \sqrt{\alpha}}{n}\sum_{i=1}^ny_t(i) = (1-\alpha \eta)x_t + \eta \sqrt{\alpha}y_t, \\
        & \bbE[y_{t+1}] = y_t - \frac{\eta}{n}y_t + \frac{\eta\sqrt{\alpha}}{n}x_t = \left(1 - \frac{\eta}{n} \right)y_t + \frac{\eta\sqrt{\alpha}}{n}x_t.
    \end{align*}
    We give an upper bound for $\bbE[x_t]$ and $\bbE[y_t]$ by the following induction. If condition
    \begin{equation} \label{eq: induction-SGD-2}
        \bbE[x_{t}] \leq 1, \qquad \bbE[y_t] \leq \sqrt{\alpha}
    \end{equation} 
    holds for $t$, then the above condition also holds for $t+1$:
    \begin{align*}
        \bbE[x_{t+1}] & \leq (1 - \alpha \eta) + \eta\sqrt{\alpha} \cdot \sqrt{\alpha} = 1 - \frac{\eta}{\eta T} + \frac{\eta }{\eta T}  \leq 1, \\
        \bbE[y_{t+1}] & \leq \left(1 - \frac{\eta}{n} \right) \sqrt{\alpha} + \eta \frac{\sqrt{\alpha}}{n} = \sqrt{\alpha}.
    \end{align*}
    Then by induction we conclude that \eqref{eq: induction-SGD-2} is true. For any $t \in [T]$ and $T \geq 2$, the lower bound for $x_t$ is much simpler to compute under our choice of parameter $\alpha = C/(\eta T)$:
    \begin{align*}
        \bbE[x_{t+1}] & \geq (1 - \alpha \eta)  x_{t} \geq (1 - \alpha \eta )^tx_1 = 4^{-t/T} \geq 4^{-C}.
    \end{align*}
    Hence $\bbE[\bar{x}_T] = \frac{x_1}{T} + \frac{1}{T} \sum_{t=2}^T \bbE[x_t|w_{t-1}] = \Theta(1)$. This then allows us to lower bound $y$ at iteration $t \in [T]$:
    \begin{align*}
        \bbE[y_{t}] & \geq \left(1-\frac{\eta}{n}\right) y_{t-1} + \frac{\eta\sqrt{\alpha}}{4^Cn} \\
        & \geq \frac{\eta\sqrt{\alpha}}{4^Cn} \cdot \left( 1 + (1-\eta/n) + \cdots (1-\eta/n)^{t-1} \right) \\
        & \geq \frac{\eta\sqrt{\alpha}}{4^Cn} \cdot \frac{1 - (1-\eta/n)^{t}}{1 - (1 - \eta/n)} .
     \end{align*}
     Now, we discuss two cases: $\eta T = \bigO{n}$ and $\eta T = \bigOmega{n}$.
    \paragraph{Case $\eta T = \bigO{n}$.} We decompose $t = n \cdot \tfrac{t}{n}$ and obtain
    \begin{align*}
        \bbE[y_t] & \geq \frac{\eta\sqrt{\alpha}}{4^Cn} \cdot \frac{1 - (1-\eta/n)^t}{1 - (1 - \eta/n)} = \frac{\eta\sqrt{\alpha}}{4^Cn} \cdot \frac{1 - (1-\eta/n)^{\tfrac{t}{n} \cdot n}}{1 - (1 - \eta/n)} \\
        & \overset{\text{(A)}}{\geq} \frac{\eta \sqrt{\alpha}}{4^C}\left( \frac{t}{n} - \frac{\eta t^2}{2n^2}\right)  \overset{\text{(B)}}{=} \frac{\eta t \sqrt{\alpha}}{2\cdot4^Cn} = \sqrt{\frac{\eta}{CT}} \cdot \frac{t}{2\cdot4^Cn}
    \end{align*}
    where $\text{(A)}$ is due to Taylor expansion, $\text{(B)}$ is due to the condition $\eta t \leq \eta T = \bigO{n}$ and $\alpha = C/(\eta T)$. We then calculate the average output
    \begin{align*}
        \bbE[\bar{y}_T] = \frac{1}{T} \sum_{t=1}^T \bbE[y_t] = \frac{1}{T} \sum_{t=1}^T \sqrt{\frac{\eta}{CT}} \cdot \frac{t}{2\cdot4^Cn} \geq \frac{1}{4\cdot 4^Cn} \cdot \sqrt{\frac{\eta T}{C}}.
    \end{align*}
    We return to the original $f(w,z)$ by inserting the above analysis on the $j$-th component:
    \begin{align*}
        \bbE[F(W_{\sgd})] \geq \bbE[G(w^{(j)}_{\sgd})]  \geq G(\bbE[w_{\sgd}^{(j)}]) & = \frac{1}{n} \sum_{i=1}^n \left(\frac{x^{(j)}_{\sgd}}{\sqrt{\eta T}} - y^{(j)}_{\sgd}(i)\right)^2 \\
        & \geq \bigOmega{\max\left\{ \frac{1}{\eta T,} \ \frac{\eta T}{n^2 }\right\}} = \bigOmega{\frac{\eta T}{n^2 }}
    \end{align*}
    where the second inequality comes from Jensen's inequality and the last inequality is due to the fact $4^{-C} \leq x_{\gd} \leq 1$. We can always choose a proper $C$ such that the difference is non-vanishing.
    \paragraph{Case $\eta T = \bigOmega{n}$.} We can directly lower bound $\bbE[y_t]$ as
    \begin{align*}
        \bbE[y_t] \geq \frac{\eta \sqrt{\alpha}}{4^C n} \cdot \frac{n}{2\eta} = \frac{1}{2\cdot4^C\sqrt{C \eta T}} = \bigOmega{\frac{1}{\sqrt{\eta T}}}
    \end{align*}
    since $(1 - \eta/n)^t \leq 1/2$ when $\eta T = \bigOmega{n}$. We then calculate the average output
    \begin{align*}
        \bbE[\bar{y}_T] = \frac{1}{T} \sum_{t=1}^T \bbE[y_t] = \sum_{t=1}^T \bigOmega{\frac{1}{\sqrt{\eta T}}} \geq \bigOmega{\frac{1}{\sqrt{\eta T}}}.
    \end{align*}
    Similarly, we return to the original $f(w,z)$ by inserting the above analysis on the $j$-th component and obtain a non-vanishing lower bound by choosing proper $C$:
    \begin{align*}
       \bbE[F(W_{\sgd})] \geq \bbE[G(w^{(j)}_{\sgd})]  & \geq G(\bbE[w_{\sgd}^{(j)}]) = \frac{1}{n} \sum_{i=1}^n \left(\frac{x^{(j)}_{\sgd}}{\sqrt{\eta T}} - y^{(j)}_{\sgd}(i)\right)^2 \geq \bigOmega{\frac{1}{\eta T}}.
    \end{align*}
    This completes our proof.
\end{proof}

\section{Missing Proofs in Section~\ref{sec: infinite}} \label{appendix: infinite}
Here we provide Lemma~\ref{lemma: dim-1-ub-gd}, the GD version of Lemma~\ref{lemma: dim-1-ub-sgd}: in dimension one, GD is also able to achieve $\bigO{1/n}$ sample complexity under the regime $T = \bigOmega{n}$.
\begin{lemma} [Restated Lemma~\ref{lemma: dim-1-ub-gd}]
In dimension one, if $f(w,z)$ is convex, $1$-smooth and realizable with $z \sim D$, then for every $\eta = \Theta(1)$, there exists $T_0 = \Theta(n)$ such that for $T \geq T_0$, the output $w_{\gd}$ of GD satisfies
\begin{equation*}
    \bbE[ F(w_{\gd})] - F(w^*) = \bigO{\frac{1}{n}}.
\end{equation*} 
\end{lemma}
\begin{proof}
From Theorem~10 in \citet{nikolakakis2022beyond}, it holds that for realizable cases (we rescale it to $f(w^*, z) = 0$ for each $z$) with step-size $\eta =  \Theta(1)$, it holds that
\begin{equation}
    \bbE[F(w_{\gd})] = \bigO{\frac{1}{T_0} + \frac{1 + T_0/n}{n}}.
\end{equation}
Therefore, for $T_0= \Theta(n)$, it holds that
\begin{equation*}
    \bbE[F(w_{\gd})] = \bigO{\nfrac{1}{n}}.
\end{equation*}
For SGD, the iteration formulates the iterate 
\begin{equation*}
    w_{t+1} = w_t - \eta \nabla F_S(w_t).
\end{equation*}
Under the realizable and convex assumption, the iteration becomes
\begin{equation*}
    w_{t+1} - w^* = (1 - \eta \nabla^2 F_S(\xi)) (w_t - w^*),
\end{equation*}
using mean value theorem, where $\xi$ is a point between $w_t$ and $w^*$. This indicates that the distance $w_t - w^*$ shrinks in each step. Due to the convexity of $F$, it holds that 
$F(w_{t+1}) \leq F(w_t)$. In summary, for any $T \geq T_0$, it holds that 
\begin{equation*}
    \bbE[F(w_{\gd})] - F(w^*) = \bigO{\nfrac{1}{n}}.
\end{equation*} 
which is the desired result.
\end{proof}

\section{Minor Proofs} \label{appendix: minor}
\subsection{Lower Bound of Term \texorpdfstring{$1/n$}{1/n}} \label{appendix: lb-sample}
\begin{lemma} \label{lemma: lb-sample}
    For every $\eta > 0$, $T > 1$, there exists a convex, $1$-smooth and realizable $f(w, z): \bbR^{2n} \to \bbR$ for every $z \in \cZ$, and a distribution $D$ such that, it holds for the output of any gradient-based algorithm $\cA[S]$
    \begin{align*}
        \bbE[F(\cA[S])] - F(w^*) = \bigOmega{\nfrac{1}{n}}.  
    \end{align*}
\end{lemma}
\begin{proof}
    We consider the following instance
    \begin{equation*}
        f(w, z = i) = \frac{1}{2} w(i)^2, \qquad z \sim \textrm{Uniform}([2n]),
    \end{equation*}
    where $w(i)$ denotes the $i$-th coordinate of $w$. Then the population risk is
    \begin{align*}
        F(w) = \bbE_{z\sim\text{Uniform}([2n])}[f(w,z)] = \frac{1}{4n} \|w\|^2,
    \end{align*}
    which achieves minimum at $w^* = 0$. It is easy to check that $f(w,z)$ is $1$-smooth, convex and realizable. Now consider any dataset $S$ of $n$ samples. Since  $z \sim \textrm{Uniform}([2n])$, with probability $\Omega(1)$,  $\Theta(n)$ coordinates are not observed. For any gradient-based algorithm with initialization $w_0 = \frac{1}{\sqrt{2n}} \cdot \bm{1}_d$, the unobserved $\Theta(n)$ coordinates will remain unchanged for any step-size $\eta$ and $T$. Then we have the following lower bound:
    \begin{align*}
        \bbE[F(\cA[S])] - F(w^*)  = \Omega\left(\frac{1}{4n} \cdot \frac{n}{2n}\right) = \Omega\left(\frac{1}{n}\right),
    \end{align*}
    which is the desired result.
\end{proof}

\subsection{Lower Bound of Term \texorpdfstring{$1/\eta T$}{1/eta T}}  \label{appendix: lb-suboptimality}
\begin{lemma} \label{lemma: lb-suboptimality}
For every $\eta > 0$, $T > 1$, there exists a convex, $1$-smooth and realizable $f(w, z): \bbR^2 \to \bbR$ for every $z \in \cZ$, and a distribution $D$ such that, the output $w_{\gd}$ for GD satisfies
\begin{align*}
 \bbE[F(w_{\gd})] - F(w^*) = \Omega\left(\frac{1}{\eta T}\right). 
\end{align*}
The same result also holds for SGD.
\end{lemma}
\begin{proof}
    As usual we suppose $\eta T \geq 1$. We define the deterministic convex and $1$-smooth function as
    \begin{align}
        f(w) = \frac{1}{2} w^2(1) + \frac{\lambda}{2} w^2(2)
    \end{align}
    with $0 < \lambda < 1$, $w(1)$ and $w(2)$ are the value of first and second coordinate of $w$. Then GD formulates the iteration
    \begin{align*}
        w_{t+1} = w_t - \eta \nabla f(w_t),
    \end{align*}
    with initialization $w_1 = (1,1)$. This is then precisely:
    \begin{align*}
        w_{t+1}(1) & = (1 -\eta) w_{t+1}(1), \qquad w_{t+1}(2) = (1 -\lambda\eta) w_{t+1}(2).
    \end{align*}
    With $\lambda = \tfrac{1}{\eta T}$, we can upper bound for $t \in [T-1]$:
    \begin{align*}
        w_{t+1}(1) \geq \frac{1}{4} e^{-\eta t} \cdot x_1(1), \qquad w_{t+1}(2) \geq \frac{1}{4} e^{- \lambda \eta t} \cdot x_1(2) = \frac{e^{-t/T}}{4}.
    \end{align*}
    The averaged output is lower bounded as
    \begin{align*}
        \bar{w}_T(2) = \sum_{t=1}^T \bar{w}_t(2) \geq \sum_{t=1}^T \frac{e^{-(t-1)/T}}{4T} \geq \frac{1}{4e} > \frac{1}{12}.
    \end{align*}
    where the second inequality is due to the fact
    \begin{align*}
        1 > e^{-1/T} > \cdots > e^{-t/T}  > \cdots > e^{-T/T} = e^{-1} 
    \end{align*}
    for any $t \in [T]$. Therefore, the suboptimality is 
    \begin{equation}
        f(w_{\gd}) - f(w^*) \geq \frac{\lambda}{2} |w_{\gd}(2)|^2 \geq \frac{1}{288\eta T}.
    \end{equation}
    Then the following result holds:
    \begin{equation*}
        \bbE[F(w_{\gd})] - F(w^*) = f(w_T) - f(w^*) \geq \Omega\left(\frac{1}{\eta T}\right)
    \end{equation*}
    because $f(w)$ is a deterministic function. Since the instance is deterministic, then the suboptimality lower bound $\bigOmega{\frac{1}{\eta T}}$ also holds for SGD.
\end{proof}

\subsection{GD Upper Bound for Realizable Smooth SCO} \label{appendix: gd-ub}
Here we derive the upper bound for GD under realizable smooth SCO, as in Table~\ref{tab: summary}. The derivation is based on Theorem~10 in \citet{nikolakakis2022beyond}.
In the realizable cases, it holds that (see \citet{nikolakakis2022beyond} for the notations):
\begin{equation*}
    \begin{split}
        &\epsilon_{\text{opt}} = \frac{\bbE \| w_1 - w_S^*\|^2}{\eta T} \\
        &\epsilon_{\text{path}} = \beta (\bbE \| w_1 - w_S^*\|^2 + \epsilon_{{\boldsymbol{c}}} \eta T) \\
        &\epsilon_{{\boldsymbol{c}}} = 0 
    \end{split}
\end{equation*}
Plug them in Theorem~10, it holds that for some constant c,
\begin{equation*}
    |\epsilon_{\text{gen}}| \leq c \frac{\beta \bbE \| w_1 - w_S^*\|^2}{n} + \frac{\beta^2 \eta T \bbE \| w_1 - w_S^*\|^2}{n^2}.
\end{equation*}
Combined with the optimization upper bound $\bigO{1/{\eta T}}$, we obtain the upper bound
\begin{equation*}
    \bbE[F(w_{\gd})] - F(w^*) = \bigO{\frac{1}{\eta T} + \frac{1}{n} + \frac{\eta T}{n^2}}.
\end{equation*}

\vskip 0.2in
\bibliography{bib}

\end{document}